\documentclass{article} 
\usepackage{iclr2027_conference,times}

\usepackage{amsmath,amsfonts,bm}

\def\eqref#1{equation~\ref{#1}}

\def\1{\bm{1}}

\DeclareMathAlphabet{\mathsfit}{\encodingdefault}{\sfdefault}{m}{sl}
\SetMathAlphabet{\mathsfit}{bold}{\encodingdefault}{\sfdefault}{bx}{n}

\usepackage{hyperref}
\usepackage{url}

\usepackage{graphicx}
\usepackage[capitalize,noabbrev]{cleveref}
\usepackage{enumitem}
\usepackage{pifont}
\usepackage{booktabs}
\usepackage{wrapfig}
\usepackage{bbding}
\usepackage{multirow}

\title{Semantic Purification for Conditional Representation Learning}

\author{
Jiaquan Wang\textsuperscript{1},
Yan Lyu\textsuperscript{1},
Chen Li\textsuperscript{2} \&
Yuheng Jia\textsuperscript{1}\thanks{%
Corresponding author. Email: \texttt{yhjia@seu.edu.cn}.%
} \\
\textsuperscript{1}School of Computer Science and Engineering,
Southeast University, China \\
\textsuperscript{2}Department of Biochemistry and Molecular Biology,
Monash University, Australia
}

\iclrfinalcopy %
\begin{document}

\maketitle
\lhead{arXiv preprint}

\begin{abstract}
Conditional representation learning aims to extract criterion-specific features for customized tasks. Recent methods construct conditional subspaces spanned by criterion-specific text bases in the embedding space of vision-language models (VLMs). Image embeddings are then projected onto these subspaces to obtain conditional representations. However, since VLMs are not explicitly trained to disentangle semantics associated with different criteria, the corresponding conditional subspaces remain coupled. This coupling induces semantic leakage during projection, thereby degrading the semantic purity of conditional representations. To suppress semantic leakage, we propose Semantic Purification for Conditional Representation Learning (SP-CRL). Specifically, SP-CRL first decomposes the original text basis and performs curvature-based adaptive truncation on the resulting basis vectors to construct a purer conditional subspace. It then identifies an appropriate noise subspace and projects image embeddings onto its null space to remove irrelevant semantic components. Extensive experiments across customized clustering, customized few-shot classification, and customized retrieval tasks demonstrate that SP-CRL achieves state-of-the-art performance with superior generalization. 
\end{abstract}

\section{Introduction}
Representation learning is a fundamental task in machine learning that extracts meaningful abstractions from raw data. Enabled by techniques such as contrastive learning \citep{chen2020simple, he2020momentum, radford2021learning} and masked prediction \citep{devlin2019bert, he2022masked, wei2022masked}, representation learning has been successfully applied to multiple downstream tasks \citep{mo2022simple, gordo2016deep, sarafianos2019adversarial, yan2023implicit}. However, existing representation learning methods do not explicitly balance distinct semantic factors. As a result, the learned representations tend to be dominated by salient semantic factors (e.g., category), while less salient factors (e.g., count) are overshadowed. As shown in \cref{fig_1}(d), this limitation constrains the applicability of representation learning to customized tasks with  varying criteria.

Consequently, conditional representation learning has emerged as a promising paradigm that extracts criterion-specific features tailored to particular requirements. Existing methods can be broadly categorized into supervised, prompt-driven, and projection-based methods. Supervised methods \citep{geng2025personalized,liu2024interactive} learn conditional representations by retraining models with criterion-specific supervision signals, incurring prohibitive annotation costs. Prompt-driven methods \citep{kwon2024image, wang2026esmc} elicit conditional representations from MLLMs via criterion-specific prompts, but repeated querying of MLLMs leads to substantial computational overhead.

\begin{figure}[ht]
  \begin{center}
    \centerline{\includegraphics[width=\textwidth]{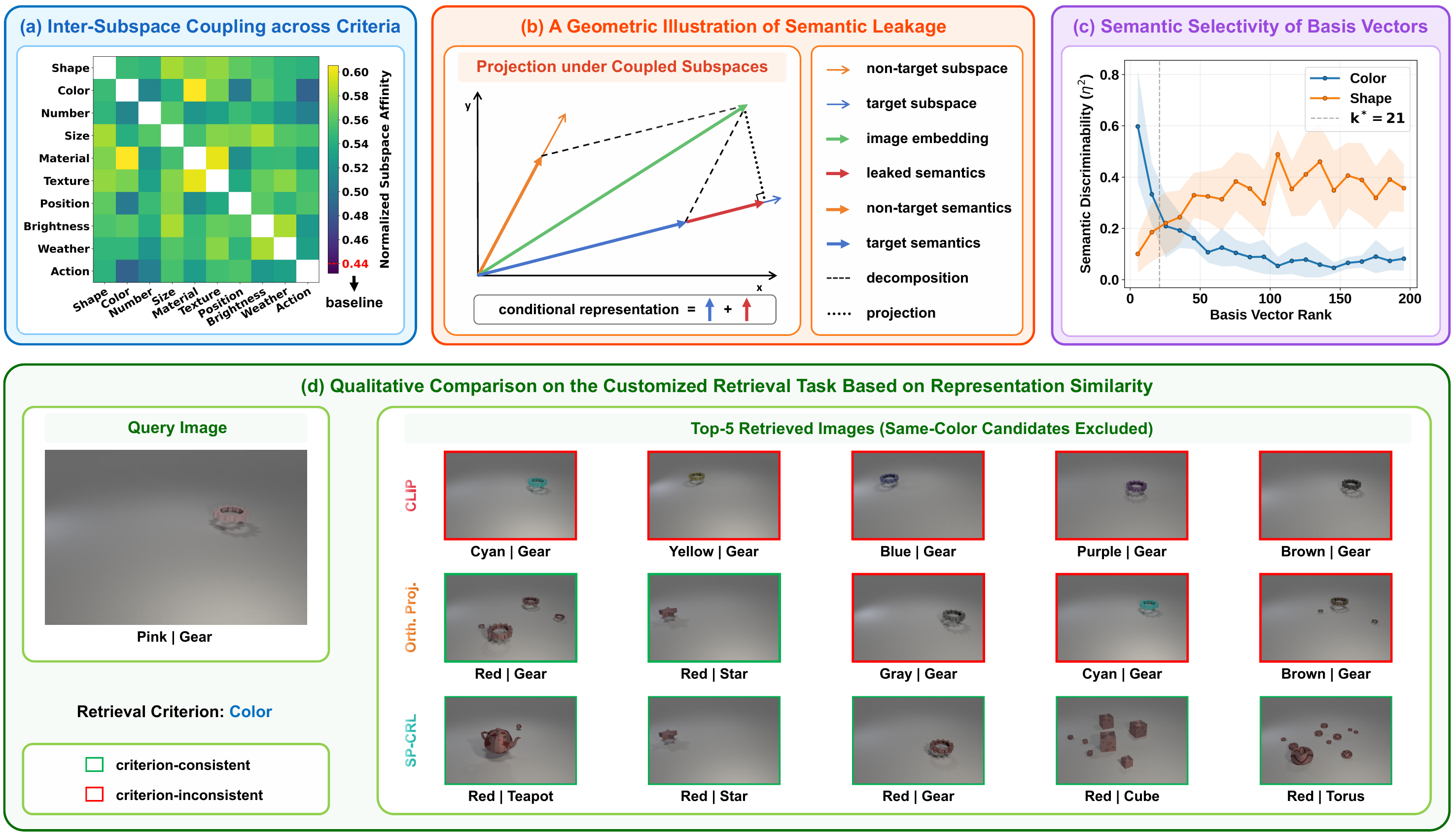}}
    \caption{\textbf{(a) Inter-subspace coupling across criteria.} We use normalized subspace affinity \citep{soltanolkotabi2014robust} to quantify coupling between conditional subspaces. The red dashed line marks the expected affinity of randomly oriented subspaces with matched dimensionality. All pairwise affinities exceed this baseline, demonstrating that coupling is inherent and pervasive among conditional subspaces.
    \textbf{(b) A geometric illustration of semantic leakage.} Under coupled subspaces, the non-target semantic component possesses a non-zero projection onto the target subspace, introducing undesired semantics into the final conditional representation.
    \textbf{(c) Semantic selectivity of basis vectors.} As singular values decrease, basis vectors exhibit decreasing semantic discriminability for the target criterion (i.e., Color) but increasing semantic discriminability for the non-target criterion (i.e., Shape). The dashed line denotes the curvature-based truncation boundary. Further experimental details are provided in Appendix \ref{sec:semantic_selectivity}.
    \textbf{(d) Qualitative comparison on the customized retrieval task based on representation similarity.} For color-based retrieval, CLIP \citep{radford2021learning} is dominated by shape semantics and retrieves only gear-shaped samples with colors substantially different from the query. Orthogonal projection reduces the influence of shape semantics but still retrieves several gear-shaped, color-dissimilar samples. In contrast, SP-CRL effectively suppresses semantic leakage, retrieving shape-diverse samples with colors more similar to the query. }
    \vspace{-1em}
    \label{fig_1}
  \end{center}
\end{figure}

Recently, projection-based methods \citep{yao2024customized, yao2024multi, liu2026conditional} have been employed to achieve efficient conditional representation learning. Specifically, LLM-generated descriptions are first encoded through a VLM's text encoder to construct criterion-specific text bases. Image embeddings are then projected onto the conditional subspaces spanned by these bases to obtain conditional representations. Unlike prompt-driven methods, projection-based methods invoke the LLM only once to construct criterion-specific text bases, avoiding repeated LLM queries at inference time.

However, existing projection-based methods focus on the conditional subspace corresponding to the target criterion, overlooking its coupling with other conditional subspaces. Since VLMs lack explicit semantic disentanglement constraints during training, text bases associated with different criteria may share semantic directions in the VLM embedding space. As illustrated in \cref{fig_1}(a), the resulting conditional subspaces consistently exhibit coupling, regardless of their semantic relatedness. Given such coupling, even standard orthogonal projection is insufficient to obtain semantically pure conditional representations: non-target semantic components possess non-zero projections onto the target subspace (\cref{fig_1}(b)). We refer to this phenomenon as semantic leakage, as it introduces undesired semantics into the final conditional representations, compromising their semantic purity.

In this paper, we propose Semantic Purification for Conditional Representation Learning (SP-CRL), which consists of two components: Criterion-Specific Basis Refinement (CSBR) and Null-Space Denoising Projection (NSDP). 
CSBR employs singular value decomposition to decompose the LLM-generated text basis into a set of orthonormal basis vectors, ranked by their singular values. As illustrated in \cref{fig_1}(c), these basis vectors exhibit varying semantic selectivity: those associated with larger singular values show higher semantic discriminability for the target criterion, while those with smaller singular values are relatively more discriminative for the non-target criterion. Motivated by this observation, CSBR applies curvature-based adaptive truncation to retain criterion-specific basis vectors, which span a purer conditional subspace with reduced semantic leakage.
To further suppress residual semantic leakage, NSDP identifies an appropriate noise subspace by greedily incorporating non-target subspaces based on the potential denoising gain and loss. It then projects the image embeddings onto the null space of the noise subspace, explicitly removing irrelevant semantic components prior to conditional representation extraction.
Both the qualitative retrieval visualization (\cref{fig_1}(d)) and the quantitative linear probing analysis (Appendix \ref{app:semantic_leakage}) show that SP-CRL effectively suppresses semantic leakage while largely preserving target semantics.

In summary, our contributions are as follows:
\begin{itemize}[leftmargin=*]
\item We reveal that conditional subspaces associated with different criteria exhibit inherent and pervasive coupling. Such coupling induces semantic leakage during projection, thereby degrading the semantic purity of conditional representations.
\item We propose Semantic Purification for Conditional Representation Learning, a projection-based conditional representation learning framework that suppresses semantic leakage through criterion-specific basis refinement and null-space denoising projection.
\item Extensive experiments show that SP-CRL achieves state-of-the-art performance on diverse tasks, including customized clustering, customized few-shot classification, and customized retrieval.
\end{itemize}

\section{Related Work}
\subsection{General Representation Learning}
General representation learning aims to extract general-purpose features from raw data. As a classical unsupervised learning method, autoencoders \citep{hinton2006reducing} obtain compact representations by minimizing reconstruction error. Denoising autoencoders \citep{vincent2008extracting} and variational autoencoders \citep{kingma2013auto} introduce corruption robustness and probabilistic modeling respectively, improving latent space stability and structure. In recent years, the self-supervised learning paradigm has expanded the boundaries of representation learning. Pretext tasks such as patch and rotation prediction \citep{doersch2015unsupervised, gidaris2018unsupervised}, solving jigsaw puzzles \citep{noroozi2016unsupervised}, and colorization \citep{zhang2016colorful} enable models to capture global-local consistency. More recently, contrastive learning frameworks \citep{chen2020simple, he2020momentum} have employed instance discrimination to learn highly discriminative representations without manual annotation. Vision-language models \citep{radford2021learning, li2023blip, zhai2023sigmoid} also employ image-text contrastive objectives to align visual and textual representations. However, existing methods tend to overemphasize salient semantic factors while underrepresenting less salient ones, limiting their applicability to customized tasks with varying criteria.

\subsection{Conditional Representation Learning}
Conditional representation learning aims to extract criterion-specific features from raw data. Supervised methods \citep{geng2025personalized, liu2024interactive} provide a straightforward solution by retraining models with criterion-specific supervision signals. However, this approach is fundamentally constrained by prohibitive annotation costs. To mitigate reliance on supervision signals, recent methods leverage MLLMs to elicit conditional representations through carefully designed prompts. IC$\mid$TC \citep{kwon2024image} obtains conditional labels for images via direct dialogue with MLLMs, while ESMC \citep{wang2026esmc} selects intermediat hidden states of MLLMs that are most responsive to the target criterion as conditional representations. Nevertheless, the substantial computational overhead caused by repeated MLLM invocations during inference remains a significant bottleneck. Projection-based methods obtain conditional representations by projecting image embeddings onto subspaces spanned by criterion-specific text bases. Multi-Map \citep{yao2024multi} and Multi-Sub \citep{yao2024customized} learn a set of trainable text basis weights for each image, which encode criterion-relevant semantics within the conditional subspace. CRL \citep{liu2026conditional} directly employs cosine similarity between image embeddings and text basis vectors as conditional representations. However, geometric coupling among conditional subspaces induces semantic leakage during projection, degrading the semantic purity of conditional representations.

\begin{figure}[ht]
  \begin{center}
    \centerline{\includegraphics[width=\textwidth]{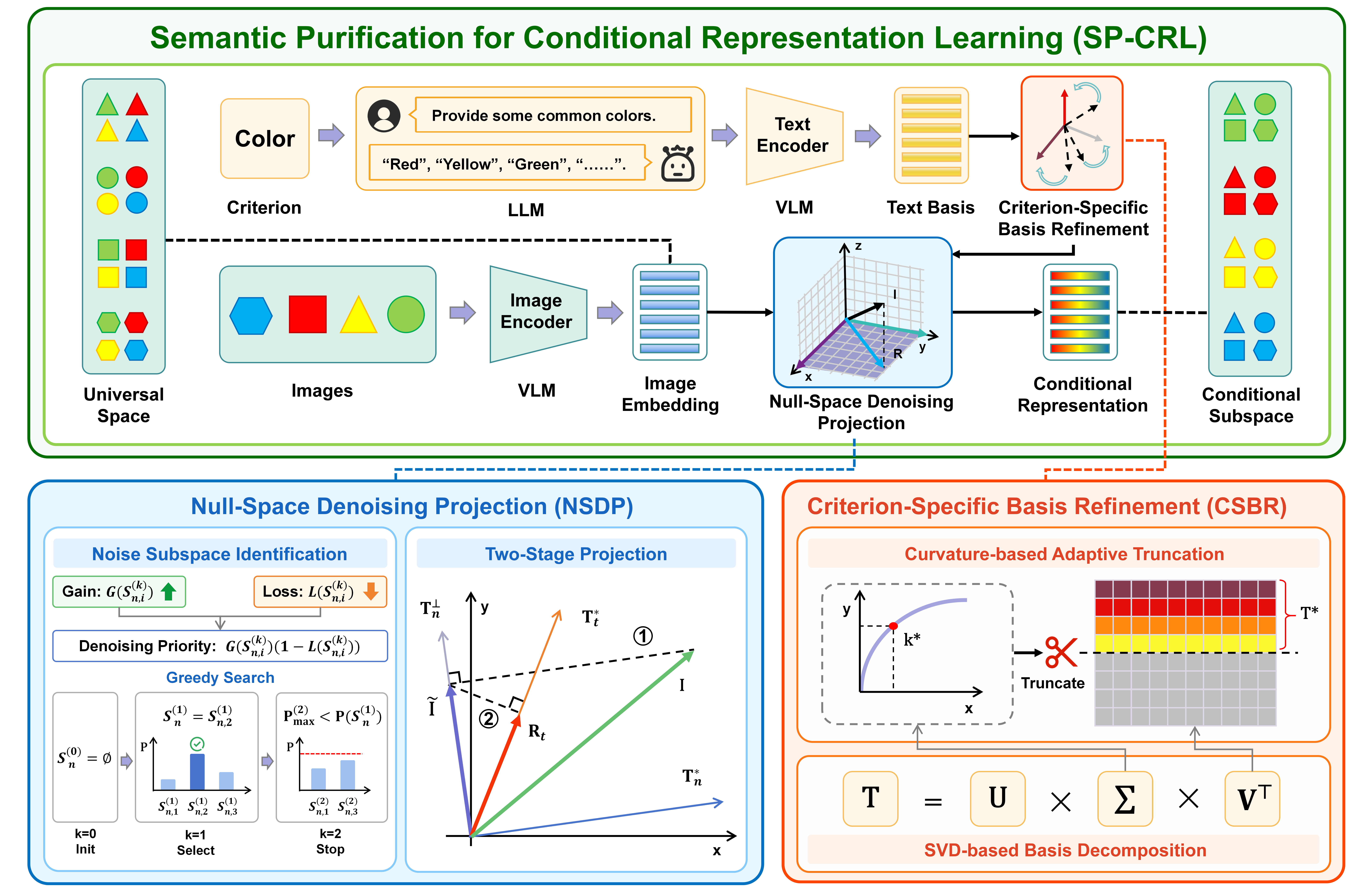}}
    \caption{The framework of SP-CRL. In CSBR, the original text basis is decomposed via SVD into $\mathbf{U}$, $\mathbf{\Sigma}$, and $\mathbf{V}^\top$, where the top $k^*$ right singular vectors are retained as the refined orthogonal basis $\mathbf{T}^*$. The truncation boundary $k^*$ is determined by the maximum curvature of the cumulative energy curve. In NSDP, the candidate noise subspace with the highest denoising priority is identified as the appropriate noise subspace $\mathcal{S}_n$. The image embeddings $\mathbf{I}$ are first projected onto the null space of the identified noise subspace to obtain the denoised embeddings $\tilde{\mathbf{I}}$ (\ding{192}), which are then projected onto the target subspace to extract the conditional representations $\mathbf{R}_t$ (\ding{193}).}
    \label{fig_2}
  \end{center}
\end{figure}

\section{Method}
\subsection{Formulation and Framework}
Given an image set $\mathcal{X}$ and its corresponding criterion set $\mathcal{C}$, conditional representation learning aims to extract representations aligned with a user-specified target criterion $C_t \in \mathcal{C}$.

As shown in \cref{fig_2}, SP-CRL projects image embeddings onto the conditional subspace to obtain conditional representations. To construct the basis of this subspace, we query an LLM to generate criterion-relevant descriptive text $W$:
\begin{equation}
W = \text{LLM}(P_1, C_t),
\end{equation}
where $P_1$  denotes the LLM prompt template conditioned on criterion $C_t$.

The generated $W$ is then encoded by a VLM to obtain a set of basis vectors $\mathbf{T}$, which span the conditional subspace corresponding to criterion $C_t$:
\begin{equation}
\mathbf{T} = \text{VLM}_{\text{text}}(P_2,W),
\end{equation}
where $P_2$ denotes the VLM prompt template that is concatenated with $W$ to construct a complete sentence. More details about the prompt templates $P_1$ and $P_2$ can be found in Appendix \ref{appendix:datasets}.

In parallel, the images in $\mathcal{X}$ are encoded through the same VLM to extract their embeddings $\mathbf{I}$:
\begin{equation}
\mathbf{I} = \text{VLM}_{\text{image}}(\mathcal{X}).
\end{equation}
Prior to projection, both the image embeddings $\mathbf{I}$ and the basis vectors $\mathbf{T}$ are $\ell_2$-normalized. The conditional representations $\mathbf{R}$ are then obtained by projecting the image embeddings $\mathbf{I}$ onto the conditional subspace spanned by $\mathbf{T}$:
\begin{equation}
\mathbf{R}=\mathbf{I}\mathbf{T}^{\top}\left(\mathbf{T}\mathbf{T}^{\top}\right)^{\dagger}\mathbf{T},
\end{equation}
where symbol $\dagger$ denote the Moore--Penrose pseudoinverse.

To suppress semantic leakage, SP-CRL introduces two complementary components: Criterion-Specific Basis Refinement (CSBR) and Null-Space Denoising Projection (NSDP), which are detailed in the following sections.

\subsection{ Criterion-Specific Basis Refinement}
Existing methods encode LLM-generated descriptions into text embeddings and directly employ them as bases for conditional subspaces. However, such text bases capture not only criterion-relevant semantics but also irrelevant semantics. To mitigate this issue, we introduce Criterion-Specific Basis Refinement, which first decomposes the original text basis into orthogonal basis vectors and then adaptively discards those dominated by irrelevant semantics.

\textbf{SVD-based Basis Decomposition.} Given the original text basis $\mathbf{T} \in \mathbb{R}^{n \times d}$, where $n$ denotes the number of descriptive texts and $d$ represents the embedding dimension, we perform SVD:
\begin{equation}
\mathbf{T} = \mathbf{U}\mathbf{\Sigma}\mathbf{V}^\top,
\end{equation}
where $\mathbf{U} \in \mathbb{R}^{n \times n}$ and $\mathbf{V} \in \mathbb{R}^{d \times d}$ are column-orthogonal matrices, and $\mathbf{\Sigma} \in \mathbb{R}^{n \times d}$ is a diagonal matrix containing the singular values in descending order.

The column vectors of $\mathbf{V}$ form an orthonormal basis for the conditional subspace. The orthogonality enables each basis vector to be independently analyzed and filtered based on its semantic selectivity.

\textbf{Curvature-based Adaptive Truncation.} Our empirical analysis (Appendix \ref{sec:semantic_selectivity}) reveals that basis vectors associated with larger singular values are more criterion-specific, with higher discriminability for criterion-relevant semantics and lower discriminability for irrelevant semantics. This semantic selectivity progressively diminishes as the singular value decreases. Based on this observation, we identify the turning point in the singular-value energy distribution and truncate the basis vectors accordingly. Specifically, we first compute the cumulative energy ratio:
\begin{equation}
E(k) = \frac{\sum_{i=1}^{k} \sigma_i^2}{\sum_{i=1}^{n} \sigma_i^2},
\end{equation}
where $\sigma_i$ denotes the $i$-th singular value. The cumulative energy curve $E(k)$ typically exhibits an ``elbow'' shape: it rises steeply as principal components capture most variance, then flattens as subsequent components contribute diminishing information.

To quantify this ``elbow" behavior, we construct a normalized discrete curve $(x_i, y_i)$ for $i = 1, 2, \ldots, n$, where $x_i = (i-1)/(n-1) \in [0, 1]$ represents the normalized proportion of retained basis vectors, and $y_i$ is the min-max normalized cumulative energy derived from $E(k)$:
\begin{equation}
y_i = \frac{E(i) - E(1)}{E(n) - E(1)} \in [0, 1].
\end{equation}
The curvature at each point is then computed as:
\begin{equation}
\kappa(x_i) = \frac{|y''(x_i)|}{(1 + (y'(x_i))^2)^{3/2}},
\end{equation}
where $y'(x_i)$ and $y''(x_i)$ are estimated via finite differences. 

The adaptive truncation boundary $k^*$ is determined by the maximum-curvature point:
\begin{equation}
k^*=\underset{i}{\arg\max}\,\kappa(x_i).
\end{equation}
The maximum-curvature point marks where the marginal energy gain begins to diminish, signaling the transition from criterion-specific basis vectors to those dominated by irrelevant semantics.

Finally, we construct the refined basis $\mathbf{T}^* $ by retaining the top $k^*$ right singular vectors:
\begin{equation}
\mathbf{T}^* = \left(\mathbf{V}_{:,1:k^*}\right)^{\top}
\in \mathbb{R}^{k^* \times d}.
\end{equation}
\subsection{Null-Space Denoising Projection}
\label{sec:nsdp}
CSBR refines the text basis to construct a purer conditional subspace, suppressing semantic leakage at the subspace level. However, due to the inherent coupling between subspaces, residual semantic leakage may still persist. We therefore propose Null-Space Denoising Projection to further suppress semantic leakage at the projection level. NSDP first identifies an appropriate noise subspace, and then projects image embeddings onto its null space to remove irrelevant semantic components.

\textbf{Noise Subspace Identification.}
Although all non-target subspaces may introduce semantic leakage, indiscriminately merging multiple of them into the noise subspace is not necessarily beneficial. A broader noise subspace increases the risk of removing target semantic components during denoising projection. Therefore, we jointly consider the potential denoising gain and loss, and employ a greedy search strategy to identify an appropriate noise subspace.

Let $\mathcal{N}^{(0)}$ denote the index set of all non-target criteria in $\mathcal{C}$. We initialize the noise subspace as $\mathcal{S}_n^{(0)}=\emptyset$. At the $k$-th iteration, each remaining non-target subspace is combined with the noise subspace selected in the previous iteration to form a candidate:
\begin{equation}
\mathcal{S}_{n,i}^{(k)}=\mathcal{S}_n^{(k-1)}+\mathcal{S}_i,\quad i\in\mathcal{N}^{(k-1)},
\end{equation}
where $\mathcal{N}^{(k-1)}$ denotes the index set of the remaining non-target criteria. For each candidate, we concatenate the bases of the involved subspaces and apply CSBR to obtain the refined basis $\mathbf{T}_{n,i}^*$.

First, we quantify the potential denoising gain as the proportion of image-embedding energy captured by each candidate noise subspace:
\begin{equation}
G(\mathcal{S}_{n,i}^{(k)})=\frac{\left\|\mathbf{I}(\mathbf{T}_{n,i}^*)^\top\mathbf{T}_{n,i}^*\right\|_F^2}{\left\|\mathbf{I}\right\|_F^2},
\end{equation}
where $\|\cdot\|_F$ denotes the Frobenius norm.

Second, we quantify the potential denoising loss as the proportion of target-subspace energy captured by each candidate noise subspace:
\begin{equation}
L(\mathcal{S}_{n,i}^{(k)})=\frac{\left\|\mathbf{T}_t^*(\mathbf{T}_{n,i}^*)^\top\mathbf{T}_{n,i}^*\right\|_F^2}{\left\|\mathbf{T}_t^*\right\|_F^2},
\end{equation}
where $\mathbf{T}_t^*$ denotes the refined basis of the target subspace.

To balance semantic leakage suppression and target semantic preservation, we define the denoising priority of each candidate noise subspace as:
\begin{equation}
P(\mathcal{S}_{n,i}^{(k)})=G(\mathcal{S}_{n,i}^{(k)})\left(1-L(\mathcal{S}_{n,i}^{(k)})\right).
\end{equation}
At each iteration, the candidate with the highest denoising priority is greedily selected as the current noise subspace:
\begin{equation}
\mathcal{S}_n^{(k)}=\arg\max_{\mathcal{S}_{n,i}^{(k)}} P(\mathcal{S}_{n,i}^{(k)}).
\end{equation}
The search continues while the denoising priority improves; otherwise, it terminates and takes the noise subspace selected in the previous iteration as the final noise subspace $\mathcal{S}_n$.

Note that the above procedure presumes non-target criteria are specified in advance. For practical scenarios where non-target criteria are unavailable, we provide an MLLM-based procedure in Appendix~\ref{app:non_target_inference} to automatically infer them for noise subspace identification.

\textbf{Null-Space Projection.}
After identifying the noise subspace, we project the image embeddings $\mathbf{I}$ onto its null space to remove the corresponding semantic components:
\begin{equation}
\tilde{\mathbf{I}} = \mathbf{I} - \mathbf{I}(\mathbf{T}_n^*)^\top\mathbf{T}_n^*.
\end{equation}
where $\mathbf{T}_n^*$ denotes the refined basis of the identified noise subspace.

\textbf{Conditional Representation Extraction.}
Finally, we extract the conditional representations by projecting the denoised image embeddings $\tilde{\mathbf{I}}$ onto the target subspace:
\begin{equation}
\mathbf{R}_t = \tilde{\mathbf{I}}(\mathbf{T}_t^*)^\top\mathbf{T}_t^*.
\end{equation}
\section{Experiments}
To evaluate the performance of the proposed SP-CRL, we apply the method to three customized tasks, including clustering, few-shot classification, and retrieval. Across these tasks, we employ diverse evaluation criteria and feature-rich datasets to enable a comprehensive assessment. We further conduct ablation studies, visualize representations (Appendix \ref{Visualization}), and perform sensitivity analyses on basis quantity (Appendix \ref{Basis_Quantity}), LLM selection (Appendix \ref{LLM_Selection}) and VLM selection (Appendix \ref{VLM_Selection}) to demonstrate the effectiveness and robustness of the proposed method.

\subsection{Customized Clustering}\label{sec1}
\textbf{Dataset.} For this task, we employ Clevr4-10k \citep{vaze2023no}, Cards \citep{yao2023augdmc} and Flowers as benchmark datasets. Clevr4-10k is a synthetic dataset containing 10,531 samples with four distinct partition criteria: ``Shape", ``Texture", ``Color", and ``Count". Cards is a poker card dataset comprising 8,029 samples partitioned according to two criteria: ``Number" and ``Suit". Moreover, we construct a new real-world flower dataset, Flowers, based on the Flowers102 dataset \citep{nilsback2008automated}, which consists of 533 samples annotated according to two criteria: ``Species" and ``Color". The detailed information of the datasets can be found in Appendix \ref{appendix:datasets}.

\textbf{Setup.} We use a pre-trained CLIP \citep{radford2021learning} model to encode images and LLM-generated text, from which SP-CRL derives conditional representations. We then perform k-means clustering on these representations, repeating the procedure 20 times with different random seeds, and report the average results. We employ three widely used clustering metrics for evaluation: Normalized Mutual Information (NMI), Accuracy (ACC), and Adjusted Rand Index (ARI).

\textbf{Performance.} As shown in \cref{tab1}, SP-CRL significantly surpasses all other customized clustering methods across all criteria, including TAC \citep{li2023image}, PCL \citep{geng2025personalized}, Multi-Map \citep{yao2024multi} and Multi-Sub \citep{yao2024customized}. When using CLIP as the backbone, SP-CRL achieves an average performance improvement of 22.21\% over CRL \citep{liu2026conditional} across all evaluation metrics. Moreover, it achieves strong performance even for the shape criterion, where traditional clustering methods such as SCAN \citep{van2020scan} typically perform best. Remarkably, SP-CRL accomplishes this without requiring any training, effectively balancing efficiency with superior performance.
 \vspace{-1em}
\begin{table}[ht]
  \caption{Performance on the task of customized clustering.}
  \label{tab1}
  \vspace{-1em}
  \begin{center}
    \begin{footnotesize}
    \setlength{\tabcolsep}{3.5pt}
        \begin{tabular}{c|ccc|ccc|ccc|ccc|c}
          \toprule
          & \multicolumn{12}{c|}{Clevr4-10k} & \\
          \cmidrule{2-13}
          Method & \multicolumn{3}{c|}{Texture} & \multicolumn{3}{c|}{Shape} & \multicolumn{3}{c|}{Color} & \multicolumn{3}{c|}{Count} & Mean \\
          \cmidrule{2-13}
          & NMI & ACC & ARI & NMI & ACC & ARI & NMI & ACC & ARI & NMI & ACC & ARI & \\ 
          \midrule
          SCAN & 0.41 & 11.97 & 0.86 & \textbf{90.99} & \underline{89.10} & \underline{84.03} & 0.20 & 11.51 & 0.01 & 3.42 & 14.29 & 1.23 & 25.67 \\
          TAC & 9.61 & 22.89 & 5.94 & 79.07 & 87.13 & 75.23 & 76.63 & 78.70 & 67.49 & 22.05 & 24.44 & 11.05 & 46.69 \\
          PCL & 9.27 & 20.03 & 4.30 & 70.25 & 74.71 & 61.51 & 74.36 & 66.93 & 58.71 & \underline{27.19} & \underline{29.06} & \underline{14.12} & 42.54 \\
          Multi-Map & 3.77 & 17.25 & 1.81 & 67.48 & 66.01 & 57.40 & 56.83 & 56.46 & 45.73 & 11.38 & 20.13 & 7.67 & 34.33 \\
          Multi-Sub & 4.93 & 18.30 & 3.52 & 72.48 & 72.07 & 61.48 & 64.34 & 62.19 & 55.62 & 13.77 & 21.64 & 9.28 & 38.30 \\
          CLIP & 1.11 & 13.09 & 0.41 & 74.22 & 73.19 & 64.15 & 0.83 & 12.23 & 0.27 & 9.50 & 19.02 & 5.70 & 22.81 \\
          CRL & \underline{10.91} & \underline{25.22} & \underline{6.55} & 78.40 & 81.83 & 71.35 & \underline{88.09} & \underline{87.26} & \underline{81.07} & 25.51 & 26.29 & 12.52 & \underline{49.58} \\
          \textbf{SP-CRL} & \textbf{14.27} & \textbf{28.27} & \textbf{9.53} & \underline{86.67} & \textbf{90.69} & \textbf{84.11} & \textbf{89.99} & \textbf{89.88} & \textbf{85.09} & \textbf{28.50} & \textbf{29.82} & \textbf{15.54} & \textbf{54.36} \\
          \midrule
          \midrule
          & \multicolumn{6}{c|}{Cards} & \multicolumn{6}{c|}{Flowers} & \\
          \cmidrule{2-13}
          Method & \multicolumn{3}{c|}{Number} & \multicolumn{3}{c|}{Suit} & \multicolumn{3}{c|}{Species} & \multicolumn{3}{c|}{Color} & Mean \\
          \cmidrule{2-13}
          & NMI & ACC & ARI & NMI & ACC & ARI & NMI & ACC & ARI & NMI & ACC & ARI & \\
          \midrule
          SCAN & 11.11 & 18.21 & \underline{17.60} & 15.01 & 32.02 & 9.48 & 53.05 & 78.04 & 47.08 & 38.12 & 56.66 & 30.03 & 33.87 \\
          TAC & 17.34 & 24.17 & 9.92 & 35.18 & 63.06 & 32.46 & 99.19 & \underline{99.81} & \underline{99.60} & 47.10 & 54.97 & 33.05 & 51.32 \\
          PCL & \underline{29.00} & \underline{32.61} & 16.18 & \underline{38.58} & 60.18 & 33.89 & 79.72 & 93.43 & 85.18 & 42.77 & 65.10 & 41.19 & 51.49 \\
          Multi-Map & 16.32 & 20.61 & 7.95 & 14.02 & 46.65 & 11.08 & 97.70 & 97.80 & 96.89 & 50.96 & 60.34 & 38.58 & 46.57 \\
          Multi-Sub & 19.58 & 23.27 & 10.28 & 18.76 & 49.09 & 15.24 & 98.35 & 98.84 & 98.52 & 53.17 & 64.21 & 42.92 & 49.35 \\
          CLIP & 16.84 & 18.91 & 8.44 & 16.52 & 43.74 & 12.93 & 93.88 & 91.19 & 89.67 & 32.74 & 49.64 & 19.19 & 41.14 \\
          CRL & 24.68 & 28.59 & 12.27 & 37.70 & \underline{62.08} & \underline{34.36} & \underline{99.29} & 99.16 & 99.10 & \underline{57.68} & \underline{68.94} & \underline{48.66} & \underline{56.04} \\
          \textbf{SP-CRL} & \textbf{38.78} & \textbf{44.45} & \textbf{28.40} & \textbf{41.61} & \textbf{73.22} & \textbf{42.20} & \textbf{100.0} & \textbf{100.0} & \textbf{100.0} & \textbf{63.47} & \textbf{79.93} & \textbf{59.57} & \textbf{64.30} \\
          \bottomrule
        \end{tabular}
    \end{footnotesize}
  \end{center}
\vspace{-2em}
\end{table}

\subsection{Customized Few-shot Classification}
\textbf{Dataset.} Similar to customized clustering, we continue to use the Clevr4-10k, Cards and Flowers datasets to evaluate SP-CRL's performance on customized few-shot classification.

\textbf{Setup.} Consistent with \cref{sec1}, CLIP is used to encode both images and text. We employ logistic regression from scikit-learn for few-shot learning with 1, 5, and 10 samples per class. For SP-CRL, the denoising gain is quantified using only the training split to ensure a fair comparison. We repeat each setting over 20 random training splits and report the average accuracy.

\textbf{Performance.} \cref{tab2} demonstrates the effectiveness of SP-CRL on customized few-shot classification. Compared to vanilla CLIP, SP-CRL achieves substantial improvements, with an average gain of 47.73\% across all settings. Against the strong baseline CRL, SP-CRL consistently delivers superior performance, particularly when training examples are extremely limited. Notably, for color-based classification on Clevr4-10k, SP-CRL achieves a 25.44\% improvement over CRL in the 1-shot setting. These consistent performance gains confirm SP-CRL's ability to extract discriminative conditional representations that effectively align with specific criteria.

\begin{table}[ht]
\vspace{-1.5em}
  \caption{Performance on the task of customized few-shot classification.}
  \label{tab2}
  \begin{center}
    \begin{footnotesize}
    \setlength{\tabcolsep}{3.5pt}
        \begin{tabular}{c|ccc|ccc|ccc|ccc|c}
          \toprule
          & \multicolumn{6}{c|}{Clevr4-10k} & \multicolumn{3}{c|}{Cards} & \multicolumn{3}{c|}{Flowers} & \\
          \cmidrule{2-13}
          Method &  \multicolumn{3}{c|}{Shape} & \multicolumn{3}{c|}{Color} & \multicolumn{3}{c|}{Number} & \multicolumn{3}{c|}{Color} & Mean\\
          \cmidrule{2-13}
          & 1 & 5 & 10 & 1 & 5 & 10 & 1 & 5 & 10 & 1 & 5 & 10 & \\ 
          \midrule
          CLIP & 58.16 & 83.17 & 89.47 & 26.85 & 57.33 & 70.00 & \underline{20.63} & 33.73 & 41.84 & 42.48 & 71.73 & 81.84 & 56.44\\
          CRL & \underline{58.69} & \underline{86.63} & \underline{92.28} & \underline{65.71} & \underline{89.13} & \underline{93.21} & 17.65 & \underline{44.51} & \underline{51.08} & \underline{54.55} & \underline{80.40} & \underline{87.62} & \underline{68.46}\\
          \textbf{SP-CRL} & \textbf{68.81} & \textbf{90.26} & \textbf{94.83} & \textbf{82.43} & \textbf{91.43} & \textbf{93.42} & \textbf{33.50} & \textbf{49.55} & \textbf{54.28} & \textbf{72.23} & \textbf{86.34} & \textbf{90.56} & \textbf{75.64}\\
          \bottomrule
        \end{tabular}
    \end{footnotesize}
  \end{center}
\vspace{-1.5em}
\end{table}

\subsection{Customized Retrieval}
\textbf{Dataset.} Following prior work \citep{ma2020fine}, we use DeepFashion \citep{liu2016deepfashion} as the evaluation dataset for this task. The dataset comprises 221k, 27k, and 27k images for training, validation, and testing. This benchmark encompasses five criteria: ``texture", ``fabric", ``shape", ``part", and ``style", containing 156, 218, 180, 216, and 230 distinct values.

\textbf{Setup.} For each query sample, we compute its cosine similarity with all candidate samples. The candidate samples are then ranked in descending order of similarity, and retrieval performance is evaluated accordingly. In the training-free setting, the conditional representations are directly evaluated. 
\begin{wraptable}{r}{0.55\textwidth}
\vspace{-1.5em}
\centering
\caption{Performance on the task of customized retrieval. Symbol $\ddagger$ signifies that training is conducted.}
\label{tab3}
\begin{center}
\begin{footnotesize}
\setlength{\tabcolsep}{3.5pt}
\begin{tabular}{ccccccc}
\toprule
Method & Texture & Fabric & Shape & Part & Style & Mean \\
\midrule
Triplet & 13.26 & 6.28 & 9.49 & 4.43 & 3.33 & 7.36 \\
CSN & 14.09 & 6.39 & 11.07 & 5.13 & 3.49 & 8.01 \\
ASEN & 15.13 & 7.11 & 12.39 & 5.51 & 3.56 & 8.74 \\
ASEN++ & 15.60 & 7.67 & 14.31 & 6.60 & 4.07 & 9.64 \\
RPF & 15.62 & 8.30 & 15.02 & \textbf{7.38} & 4.77 & 10.22 \\
CLIP & 9.14 & 4.68 & 7.86 & 4.26 & 4.48 & 6.08 \\
CRL & 11.03 & 6.76 & 11.80 & 5.56 & 4.42 & 7.93 \\
CRL$^\ddagger$ & \underline{15.92} & \underline{8.54} & \underline{16.33} & 6.90 & \underline{5.27} & \underline{10.59} \\
\midrule
SP-CRL & 11.05 & 6.74 & 12.47 & 5.69 & 4.95 & 8.18 \\
SP-CRL$^\ddagger$ & \textbf{16.04} & \textbf{8.67} & \textbf{16.62} & \underline{7.14} & \textbf{5.81} & \textbf{10.84} \\
\bottomrule
\end{tabular}
\end{footnotesize}
\end{center}
\vspace{-1.5em}
\end{wraptable}
In the training setting, a two-layer MLP is additionally trained with triplet loss to further optimize the representations before evaluation. For fair comparison, SP-CRL quantifies the denoising gain using only the training split in both settings. We adopt Mean Average Precision (mAP) as the evaluation metric.

\textbf{Performance.} \cref{tab3} presents the performance comparison on customized retrieval. Equipped with merely a simple MLP, SP-CRL achieves the best mean mAP of 10.84\%. This result outperforms previous methods, including Triplet \citep{veit2017conditional}, CSN \citep{veit2017conditional}, ASEN \citep{ma2020fine}, ASEN++ \citep{dong2021fine}, RPF \citep{dong2023region} and CRL. Without training, SP-CRL already surpasses vanilla CLIP by 34.54\% in mean mAP, validating its effectiveness in generating discriminative criterion-specific representations. 

\subsection{Ablation Study}
To evaluate SP-CRL and its key design choices, we conduct ablation studies on the customized clustering task using the Clevr4-10k and Cards datasets. 

\textbf{Ablation on Components.}
As shown in \cref{tab4}, we use standard orthogonal projection as the baseline to evaluate the contributions of CSBR and NSDP. On the Clevr4-10k dataset, CSBR alone improves the color-based clustering NMI from 73.91\% to 87.87\% by constructing a more criterion-specific basis, while NSDP alone increases the NMI to 78.95\% by removing irrelevant semantic components in the identified noise subspace. Combining CSBR and NSDP further improves the NMI to 89.99\%, achieving the best performance. Similar trends are observed on the Cards dataset, confirming the complementary effects of CSBR and NSDP.

\begin{table}[ht]
\vspace{-1.5em}
  \caption{Ablation study on components.}
  \label{tab4}
  \begin{center}
    \begin{footnotesize}
    \setlength{\tabcolsep}{3.5pt}
        \begin{tabular}{cc|ccc|ccc|ccc|ccc}
          \toprule
          \multicolumn{2}{c|}{Components} & \multicolumn{6}{c|}{Clevr4-10k} & \multicolumn{6}{c}{Cards}\\
          \midrule
          \multirow{2.5}{*}{CSBR} & \multirow{2.5}{*}{NSDP} & \multicolumn{3}{c|}{Shape} & \multicolumn{3}{c|}{Color} & \multicolumn{3}{c|}{Number} & \multicolumn{3}{c}{Suit}\\
          \cmidrule{3-14}
          ~ & ~ & NMI & ACC & ARI & NMI & ACC & ARI & NMI & ACC & ARI & NMI & ACC & ARI\\
          \midrule
          ~ & ~ & 81.72 & 82.12 & 75.00 & 73.91 & 70.78 & 63.99 & 36.86 & 39.38 & 23.27 & 35.99 & 55.41 & 28.44\\
          \Checkmark & ~ & \underline{82.62} & \underline{85.65} & \underline{77.85} & \underline{87.87} & \underline{85.06} & \underline{80.35} & \underline{38.10} & \underline{40.76} & 24.08 & \underline{39.65} & \underline{70.34} & \underline{39.00}\\
          ~ & \Checkmark & 82.48 & 82.16 & 76.14 & 78.95 & 75.68 & 68.91 & 37.33 & 39.97 & \underline{25.78} & 39.39 & 58.89 & 33.18\\
          \Checkmark & \Checkmark & \textbf{86.67} & \textbf{90.69} & \textbf{84.11} & \textbf{89.99} & \textbf{89.88} & \textbf{85.09} & \textbf{38.78} & \textbf{44.45} & \textbf{28.40} & \textbf{41.61} & \textbf{73.22} & \textbf{42.20}\\
          \bottomrule
        \end{tabular}
    \end{footnotesize}
  \end{center}
\vspace{-1em}
\end{table}

\textbf{Ablation on Basis Truncation.}
We compare the proposed curvature-based adaptive truncation with fixed-ratio basis truncation. As shown in \cref{tab:basis_truncation}, the optimal truncation ratio (i.e., the proportion of retained basis vectors) varies across different criteria, while our strategy achieves consistently strong performance. This demonstrates the advantage of our strategy in identifying an appropriate truncation boundary for different criteria.

\begin{table}[ht]
\vspace{-1.5em}
\caption{Ablation study on basis truncation.}
\label{tab:basis_truncation}
\begin{center}
\begin{footnotesize}
\setlength{\tabcolsep}{3.5pt}
\begin{tabular}{c|ccc|ccc|ccc|ccc}
\toprule
& \multicolumn{6}{c|}{Clevr4-10k} & \multicolumn{6}{c}{Cards}\\
\cmidrule{2-13}
Retention Ratio & \multicolumn{3}{c|}{Shape} & \multicolumn{3}{c|}{Color} & \multicolumn{3}{c|}{Number} & \multicolumn{3}{c}{Suit}\\
\cmidrule{2-13}
& NMI & ACC & ARI & NMI & ACC & ARI & NMI & ACC & ARI & NMI & ACC & ARI\\
\midrule
20\% & 81.38 & \underline{85.38} & 76.40 & \underline{86.86} & \underline{84.63} & \underline{79.28} & 22.57 & 26.95 & 11.85 & \underline{38.70} & \underline{67.90} & \underline{37.43}\\
40\% & 81.99 & 84.70 & 77.05 & 85.10 & 82.08 & 76.55 & 31.74 & 34.72 & 18.30 & 35.82 & 58.23 & 31.16\\
60\% & 81.89 & 82.57 & 75.68 & 80.85 & 76.96 & 70.88 & \underline{38.04} & 40.58 & 23.53 & 37.09 & 57.86 & 30.77\\
80\% & \textbf{82.78} & 84.08 & \underline{77.19} & 76.63 & 72.55 & 66.07 & 37.85 & \underline{40.66} & \underline{23.86} & 36.05 & 57.51 & 30.08\\
100\% (Full) & 81.72 & 82.12 & 75.00 & 73.91 & 70.78 & 63.99 & 36.86 & 39.38 & 23.27 & 35.99 & 55.41 & 28.44\\
Adaptive (Ours) & \underline{82.62} & \textbf{85.65} & \textbf{77.85} & \textbf{87.87} & \textbf{85.06} & \textbf{80.35} & \textbf{38.10} & \textbf{40.76} & \textbf{24.08} & \textbf{39.65} & \textbf{70.34} & \textbf{39.00}\\
\bottomrule
\end{tabular}
\end{footnotesize}
\end{center}
\vspace{-1em}
\end{table}

\textbf{Ablation on Noise Subspace Identification.}
We compare different noise subspace choices, including each non-target subspace individually and their combination, against a no-denoising baseline. Symbol $^\star$ denotes the noise subspace identified by our method. As shown in \cref{tab:noise_selection}, denoising with an inappropriate noise subspace can degrade the quality of conditional representations (e.g., using Shape as the noise subspace when Texture is the target subspace). In contrast, our identified noise subspace achieves the best performance, indicating a better balance between suppressing semantic leakage and preserving target semantics.

\newcommand{\dom}{\rlap{$^{\star}$}}

\begin{table}[ht]
\vspace{-1.5em}
\caption{Ablation study on noise subspace identification.}
\label{tab:noise_selection}
\begin{center}
\begin{footnotesize}
\setlength{\tabcolsep}{3.5pt}
\begin{tabular}{c|ccc|ccc|ccc|ccc}
\toprule
& \multicolumn{12}{c}{Clevr4-10k (Target Subspace)}\\
\cmidrule{2-13}
Noise Subspace & \multicolumn{3}{c|}{Texture} & \multicolumn{3}{c|}{Shape} & \multicolumn{3}{c|}{Color} & \multicolumn{3}{c}{Count}\\
\cmidrule{2-13}
& NMI & ACC & ARI & NMI & ACC & ARI & NMI & ACC & ARI & NMI & ACC & ARI\\
\midrule
None & \underline{13.80} & \underline{27.86} & \underline{8.60} & 82.62 & 85.65 & 77.85 & 87.87 & 85.06 & 80.35 & 25.74 & 28.56 & 14.22\\
Texture & - & - & - & 82.55 & 84.16 & 76.92 & 87.97 & \underline{88.28} & 82.04 & 26.38 & \underline{29.33} & 14.88\\
Shape & 10.45 & 24.75 & 6.05 & - & - & - & \textbf{89.99}\dom & \textbf{89.88}\dom & \textbf{85.09}\dom & \textbf{29.00} & 29.16 & \underline{15.01}\\
Color & \textbf{14.27}\dom & \textbf{28.27}\dom & \textbf{9.53}\dom & \textbf{86.67}\dom & \textbf{90.69}\dom & \textbf{84.11}\dom & - & - & - & 26.61 & 27.49 & 13.28\\
Count & 11.23 & 26.00 & 7.19 & 81.51 & 83.22 & 75.94 & 87.94 & 86.67 & 81.42 & - & - & -\\
Combination & 10.95 & 23.59 & 6.35 & \underline{84.89} & \underline{87.63} & \underline{80.66} & \underline{89.33} & 87.15 & \underline{83.09} & \underline{28.50}\dom & \textbf{29.82}\dom & \textbf{15.54}\dom\\
\bottomrule
\end{tabular}
\end{footnotesize}
\end{center}
\vspace{-1em}
\end{table}

\section{Conclusion}
In this paper, we reveal that conditional subspaces corresponding to different criteria consistently exhibit coupling. Such coupling leads to semantic leakage during projection, degrading the semantic purity of conditional representations. To address this issue, we propose Semantic Purification for Conditional Representation Learning (SP-CRL), which consists of Criterion-Specific Basis Refinement (CSBR) and Null-Space Denoising Projection (NSDP). 
Extensive experiments across diverse customized tasks, including clustering, few-shot classification, and retrieval, demonstrate that SP-CRL consistently achieves superior performance with superior generalization.



\bibliography{iclr2027_conference}
\bibliographystyle{iclr2027_conference}

\clearpage
\appendix

\section{Semantic Selectivity of SVD-Derived Basis Vectors}
\label{sec:semantic_selectivity}

To investigate the semantic selectivity of the basis vectors used in CSBR, we evaluate the discriminability of each basis vector across different semantic criteria. Given the LLM-generated text basis $\mathbf{T}\in\mathbb{R}^{m\times d}$ for a target criterion, we perform singular value decomposition:
\begin{equation}
\mathbf{T}=\mathbf{U}\mathbf{\Sigma}\mathbf{V}^{\top}.
\end{equation}
The right singular vectors $\{\mathbf{v}_j\}_{j=1}^{r}$ form an orthonormal basis for the conditional subspace and are ordered by their corresponding singular values, $\sigma_1\geq\sigma_2\geq\cdots\geq\sigma_r$.

Given image embeddings $\{\mathbf{x}_n\}_{n=1}^{N}$, we compute the projection coefficient of the $n$-th image onto the $j$-th basis vector as:
\begin{equation}
z_{n,j}=\mathbf{x}_n^{\top}\mathbf{v}_j.
\end{equation}
For each semantic criterion $C$, we quantify the semantic discriminability of the $j$-th basis vector using the explained-variance ratio $\eta^2$:
\begin{equation}
\eta_{j,c}^{2}=\frac{\operatorname{Var}_{y^{(c)}}\!\left(\mathbb{E}[z_j\mid y^{(c)}]\right)}{\operatorname{Var}(z_j)}.
\end{equation}
Here, $y^{(c)}$ denotes the labels under criterion $C$, and the variance over $y^{(c)}$ is weighted by the empirical class proportions. The ratio $\eta_{j,c}^{2}\in[0,1]$ measures the proportion of variation along the $j$-th basis vector explained by criterion $C$, where a larger value indicates stronger semantic discriminability.

We conduct the analysis on two criterion pairs, Color-Shape (Clevr4-10k dataset) and Number-Suit (Cards dataset). For each pair, we alternately treat each criterion as the target criterion and construct its corresponding text basis.

\begin{figure*}[ht]
\vspace{-0.5em}
\centering
\includegraphics[width=\textwidth]{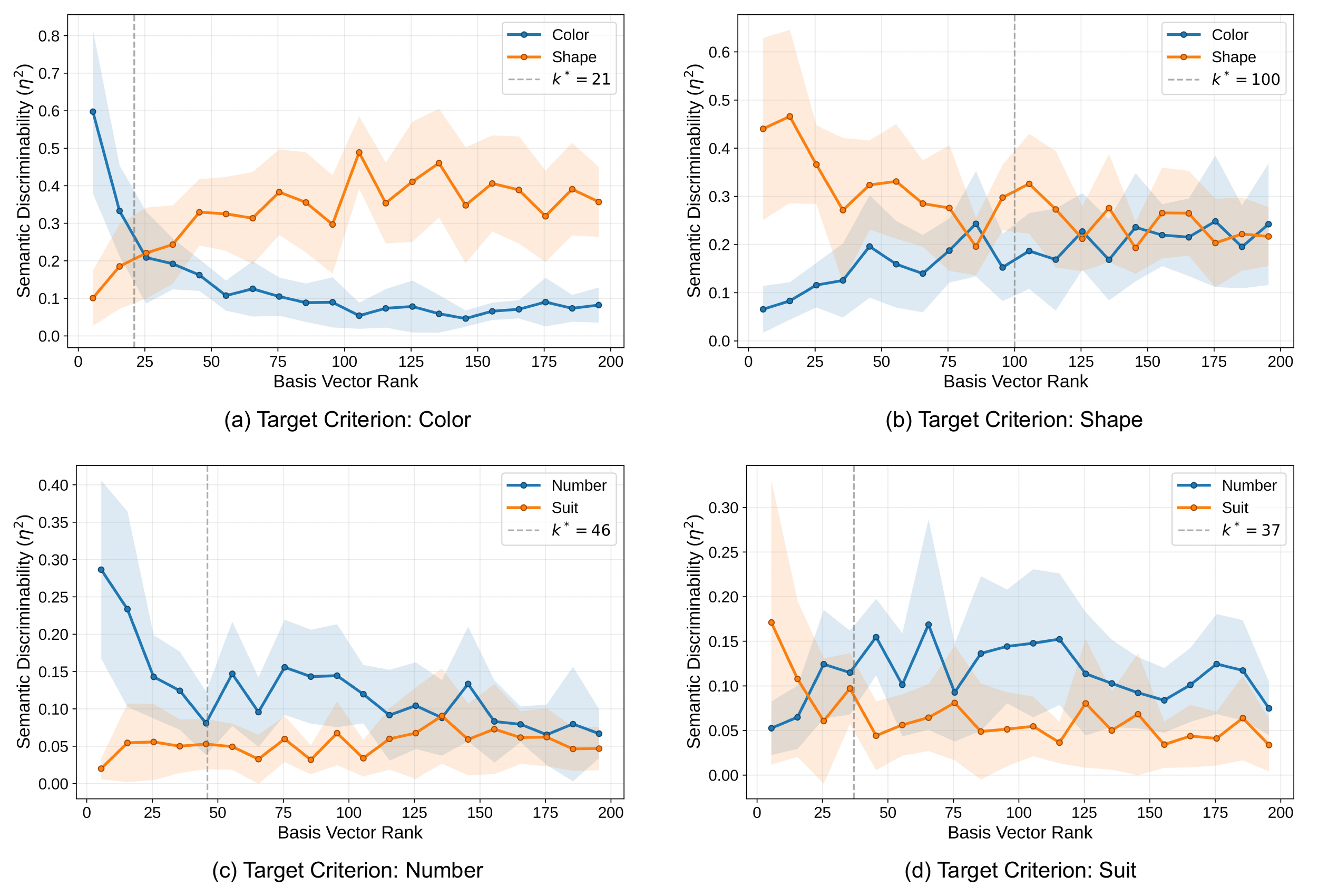}
\caption{
Semantic discriminability of SVD-derived basis vectors across different criteria.
Basis vectors are ranked by descending singular values.
Solid curves represent the mean $\eta^2$ for each bin of 10 basis vectors, while shaded regions denote the corresponding standard deviation.
Dashed vertical lines indicate the curvature-based truncation boundary $k^*$ determined by CSBR.
}
\label{fig:semantic_selectivity}
\vspace{-0.5em}
\end{figure*}

As shown in \cref{fig:semantic_selectivity}, basis vectors associated with larger singular values consistently exhibit higher semantic discriminability for the target criterion than for the non-target criterion. As the singular value decreases, the discriminability gap gradually narrows, indicating weaker semantic selectivity. These results suggest that high-singular-value basis vectors are criterion-specific, capturing more criterion-relevant semantics. In contrast, low-singular-value basis vectors capture relatively more irrelevant semantics, increasing the risk of semantic leakage in the resulting conditional subspace. Notably, the curvature-based truncation boundary $k^*$ generally lies near the region where the semantic selectivity of the basis vectors begins to diminish, supporting the effectiveness of our truncation strategy.

\section{Semantic Leakage Analysis via Linear Probing}
\label{app:semantic_leakage}

To quantitatively assess the effectiveness of SP-CRL, we conduct a linear probing analysis on the Clevr4-10k and Cards datasets. Given a representation conditioned on the target criterion, we train separate linear probes to predict the labels of the corresponding images under the target and non-target criteria, respectively. Higher prediction accuracy on the target criterion indicates better preservation of target semantics, whereas lower prediction accuracy on the non-target criterion indicates more effective suppression of semantic leakage. On the Clevr4-10k dataset, we evaluate the Color and Shape criteria, using each criterion in turn as the target criterion and the other as the non-target criterion. On the Cards dataset, we evaluate the Number and Suit criteria in the same manner. For each setting, we randomly sample 1,000 images from the corresponding dataset for training and use all remaining images for testing. We repeat this procedure with five random seeds and report the average test accuracy.

As shown in \cref{tab:semantic_leakage}, SP-CRL consistently achieves the lowest non-target prediction accuracy across all evaluated settings. On the Clevr4-10k dataset, when Color is the target criterion and Shape is the non-target criterion, SP-CRL reduces the non-target accuracy from 97.85\% with orthogonal projection to 66.38\%, while maintaining a high target accuracy of 97.52\%. When their roles are reversed, the non-target accuracy decreases from 97.77\% to 89.76\%, while the target accuracy even improves slightly from 98.21\% to 98.45\%. Similar trends are observed on the Cards dataset. These results demonstrate that SP-CRL effectively suppresses semantic leakage while largely preserving target semantics.

\begin{table}[ht]
\vspace{-0.5em}
\centering
\caption{Linear probing prediction accuracy under target and non-target criteria.}
\label{tab:semantic_leakage}
\small
\setlength{\tabcolsep}{3pt}
\begin{tabular}{llcccccc}
\toprule
\multirow{2}{*}{Dataset} & \multirow{2}{*}{Target / Non-target} & \multicolumn{2}{c}{CLIP} & \multicolumn{2}{c}{Orthogonal Projection} & \multicolumn{2}{c}{SP-CRL} \\
\cmidrule(lr){3-4}\cmidrule(lr){5-6}\cmidrule(lr){7-8}
& & Target $\uparrow$ & Non-target $\downarrow$ & Target $\uparrow$ & Non-target $\downarrow$ & Target $\uparrow$ & Non-target $\downarrow$ \\
\midrule
\multirow{2}{*}{Cards} & Number / Suit & \textbf{77.00} & 72.82 & 64.31 & 60.85 & 62.61 & \textbf{53.96} \\
& Suit / Number & \textbf{88.99} & 56.19 & 84.16 & 42.02 & 82.89 & \textbf{33.02} \\
\midrule
\multirow{2}{*}{Clevr4-10k} & Color / Shape & \textbf{98.84} & 98.98 & 98.68 & 97.85 & 97.52 & \textbf{66.38} \\
& Shape / Color & \textbf{99.07} & 98.68 & 98.21 & 97.77 & 98.45 & \textbf{89.76} \\
\midrule
\multicolumn{2}{l}{Mean} & \textbf{90.98} & 81.67 & 86.34 & 74.62 & 85.37 & \textbf{60.78} \\
\bottomrule
\end{tabular}
\vspace{-0.5em}
\end{table}

\section{Datasets and Corresponding Prompts}
\label{appendix:datasets}

In this section, we provide detailed descriptions of the datasets used in our experiments, along with the corresponding prompts for LLM and VLM. We organize the content by dataset, presenting the dataset structure and the specific prompts employed for each criterion.

\subsection{Clevr4-10k}
\label{subsec:clevr4}

\subsubsection{Dataset Structure}

Clevr4-10k \citep{vaze2023no} is a synthetic dataset built upon the CLEVR dataset. It consists of 10,531 images featuring multiple 3D objects positioned within fixed scenes. Each image in Clevr4-10k is annotated with all four criteria simultaneously, allowing for comprehensive evaluation of conditional representation learning methods.  \cref{clevr4} shows the structure of the Clevr4-10k dataset.

\begin{table}[ht]
\vspace{-0.5em}
    \centering
    \caption{Structure of the Clevr4-10k dataset.}
    \label{clevr4}
    \begin{tabular}{lcc}
        \toprule
        Criterion & Number of Classes & Class Examples \\
        \midrule
        Shape & 10 & cone, cube, cylinder, diamond, gear, ... \\
        Texture & 10 & brick, checkered, chessboard, circles, emojis, ... \\
        Color & 10 & blue, brown, cyan, gray, green, ... \\
        Count & 10 & 1, 2, 3, 4, 5, ... \\
        \bottomrule
    \end{tabular}
\vspace{-0.5em}
\end{table}

\subsubsection{LLM Prompts}
For Clevr4-10k, we design criterion-specific prompts to query the LLM for generating descriptive texts. Below we present the prompts used for each criterion: 

\small
\texttt{Please generate common expressions to describe the shape (Texture / Color / Count), as many as possible, formatted as: ["...", "...", "..."]. Ensure all items are unique and written in a single line, without any nested lists or additional formatting. Only generate the list, and do not include any additional information.}

\subsubsection{VLM Prompts}

For encoding the LLM-generated texts using the VLM, we construct prompts that combine the criterion with each descriptive text. The general template is:

\small
\texttt{Objects with the [CRITERION] of [TEXT].}

where \texttt{[CRITERION]} is replaced with the specific criterion (e.g., ``shape", ``texture", ``color", ``count") and \texttt{[TEXT]} is replaced with each generated descriptive text.

\subsection{Cards}
\label{subsec:cards}

\subsubsection{Dataset Structure}

The Cards dataset \citep{yao2023augdmc} is a poker card dataset comprising 8,029 images of playing cards. Each card is photographed under various conditions, providing natural variations in lighting, angle, and background. \cref{cards} shows the structure of the Cards dataset.

\begin{table}[ht]
\vspace{-0.5em}
    \centering
    \caption{Structure of the Cards dataset.}
    \label{cards}
    \begin{tabular}{lcc}
        \toprule
        Criterion & Number of Classes & Class Examples \\
        \midrule
        Number & 13 & Ace, Two, three, ..., Jack, Queen, King \\
        Suit & 4 & Clubs($\clubsuit$), Diamonds($\diamondsuit$), Hearts($\heartsuit$), Spades($\spadesuit$) \\
        \bottomrule
    \end{tabular}
\vspace{-0.5em}
\end{table}

\subsubsection{LLM Prompts}
For Cards, we design criterion-specific prompts to query the LLM for generating descriptive texts. Below we present the prompts used for each criterion:

\small
\texttt{Please generate common expressions to describe the number (suit) of playing cards, as many as possible, formatted as: ["...", "...", "..."]. Ensure all items are unique and written in a single line, without any nested lists or additional formatting. Only generate the list, and do not include any additional information.}

\subsubsection{VLM Prompts}

For encoding the LLM-generated texts using the VLM, we construct prompts that combine the criterion with each descriptive text. The general template is:

\small
\texttt{A poker card with the [CRITERION] of [TEXT].}

\noindent where \texttt{[CRITERION]} is replaced with the specific criterion (e.g., ``number", ``suit") and \texttt{[TEXT]} is replaced with each generated descriptive text.


\subsection{Flowers}
\label{subsec:flowers}
\subsubsection{Dataset Structure}
The Flowers dataset is a real-world flower dataset constructed in-house from the Flowers102 dataset \citep{nilsback2008automated}. \cref{flowers} shows the structure of the Flowers dataset.

\begin{table}[ht]
\vspace{-0.5em}
    \centering
    \caption{Structure of the Flowers dataset.}
    \label{flowers}
    \begin{tabular}{lcc}
        \toprule
        Criterion & Number of Classes & Class Examples \\
        \midrule
        Species & 4 & daisy, hyacinth, lily, rose \\
        Color & 5 & pink, purple, red, white, yellow \\
        \bottomrule
    \end{tabular}
\vspace{-0.5em}
\end{table}

\subsubsection{LLM Prompts}
For Flowers, we design criterion-specific prompts to query the LLM for generating descriptive texts. Below we present the prompts used for each criterion:
\small
\texttt{Please generate common expressions to describe the species (color) of flowers, as many as possible, formatted as: ["...", "...", "..."]. Ensure all items are unique and written in a single line, without any nested lists or additional formatting. Only generate the list, and do not include any additional information.}
\subsubsection{VLM Prompts}
For encoding the LLM-generated texts using the VLM, we construct prompts that combine the criterion with each descriptive text. The general template is:

\small
\texttt{Flowers with the [CRITERION] of [TEXT].}

\noindent where \texttt{[CRITERION]} is replaced with the specific criterion (e.g., ``species", ``color") and \texttt{[TEXT]} is replaced with each generated descriptive text.

\subsection{DeepFashion}
\label{subsec:deepfashion}

\subsubsection{Dataset Structure}

DeepFashion \citep{liu2016deepfashion} is a large-scale fashion dataset widely used for clothing-related computer vision tasks. Following previous works \citep{ma2020fine}, we use the category and attribute prediction benchmark split, which consists of 221,000 training images, 27,000 validation images, and 27,000 test images. Each image in DeepFashion is annotated with multiple attribute criteria simultaneously. Table~\ref{deepfashion} shows the structure of the DeepFashion dataset.

\begin{table}[h]
\vspace{-0.5em}
    \centering
    \caption{Structure of the DeepFashion dataset.}
    \label{deepfashion}
    \begin{tabular}{lcl}
        \toprule
        Criterion & Number of Classes & Class Examples \\
        \midrule
        Texture & 156 & abstract, animal, baroque, camo, circle, ... \\
        Fabric & 218 & acid, bead, canvas, cloud, feather, ... \\
        Shape & 180 & a-line, boxy, cropped, flared, shirt, ... \\
        Part & 216 & bell, bow, buttoned, cowl, flat, ... \\
        Style & 230 & baseball, bike, eagle, fox, mickey, ... \\
        \bottomrule
    \end{tabular}
\vspace{-0.5em}
\end{table}

\subsubsection{LLM Prompts}

For DeepFashion, we design criterion-specific prompts to query the LLM for generating descriptive texts. Below we present the prompts used for each criterion:

\small
\texttt{Please generate common expressions to describe the texture (fabric / shape / part / style) of fashion items, as many as possible, formatted as: ["...", "...", "..."]. Ensure all items are unique and written in a single line, without any nested lists or additional formatting. Only generate the list, and do not include any additional information.}

\subsubsection{VLM Prompts}

For encoding the LLM-generated texts using the VLM, we construct prompts that combine the criterion with each descriptive text. The general template is:

\small
\texttt{A fashion with the [CRITERION] of [TEXT].}

\noindent where \texttt{[CRITERION]} is replaced with the specific criterion (e.g., ``texture", ``fabric", ``shape", ``part", ``style") and \texttt{[TEXT]} is replaced with each generated descriptive text.

\section{Automatic Inference of Non-Target Criteria}
\label{app:non_target_inference}

The noise subspace identification described in Section~\ref{sec:nsdp} assumes that non-target criteria are specified in advance. In practical scenarios, however, such criteria may not always be available. To address this issue, we provide an MLLM-based procedure to automatically infer non-target criteria for noise subspace identification.

\textbf{Removal of Target Semantic Components.}
Given a target criterion $C_t$, we remove the components of the image embeddings that lie within the corresponding conditional subspace:
\begin{equation}
\mathbf{I}_{\mathrm{res}} = \mathbf{I} - \mathbf{I}(\mathbf{T}_t^*)^\top \mathbf{T}_t^*,
\label{eq:residual_embedding}
\end{equation}
where $\mathbf{I}$ denotes the image embeddings and $\mathbf{T}_t^*$ denotes the refined basis of the target subspace.
This operation reduces the influence of target semantics on subsequent non-target criterion inference.

\textbf{Clustering-Based Representative Selection.}
We then perform k-means clustering on the residual embeddings to reveal their semantic structures. The number of clusters $K$ is set according to the semantic diversity of the dataset. Let $\mathbf{i}_i^{\mathrm{res}}$ denote the residual embedding of the $i$-th image. The clustering objective is formulated as:
\begin{equation}
\min_{\{z_i\},\{\boldsymbol{\mu}_k\}}
\sum_{i=1}^{N}
\left\|\mathbf{i}_i^{\mathrm{res}}-\boldsymbol{\mu}_{z_i}\right\|_2^2,
\label{eq:non_target_kmeans}
\end{equation}
where $z_i$ denotes the cluster assignment of the $i$-th image and $\boldsymbol{\mu}_k$ denotes the centroid of the $k$-th cluster.

To obtain representative images from the clusters, we select the image closest to each cluster centroid:
\begin{equation}
r_k = \arg\min_{i:z_i=k}
\left\|\mathbf{i}_i^{\mathrm{res}}-\boldsymbol{\mu}_k\right\|_2^2.
\label{eq:cluster_representative}
\end{equation}
The resulting representative image set is denoted by:
\begin{equation}
\mathcal{X}_{\mathrm{rep}}
=
\{x_{r_1},x_{r_2},\ldots,x_{r_K}\}.
\label{eq:representative_set}
\end{equation}

\textbf{Non-Target Criterion Inference.}
For non-target criterion inference, we design an MLLM prompt template $P_3$ to identify salient non-target criteria from the representative images:

\small
\texttt{Given representative images from different clusters of the same dataset and the target criterion [TARGET CRITERION], please identify a small number of salient visual criteria, other than the target criterion, that explain the differences among these clusters. Ensure all items are unique and formatted as: ["...", "...", "..."]. Only generate the list, and do not include any additional information.}

Finally, the non-target criterion set is inferred by the MLLM as:
\begin{equation}
\mathcal{C}_{\mathrm{nt}}
=
\text{MLLM}(P_3,C_t,\mathcal{X}_{\mathrm{rep}}).
\label{eq:non_target_inference}
\end{equation}
The inferred non-target criteria can be used for noise subspace identification in NSDP.

\section{Visualization of Conditional Representations}
\label{Visualization}
\cref{fig_5} presents t-SNE visualizations of learned representations on the Clevr4-10k dataset under the ``Color" criterion. CLIP's general representations (a) exhibit severe entanglement across semantic categories, failing to capture criterion-specific distinctions. While CRL (b) shows improved cluster formation, considerable inter-cluster overlap persists. In contrast, SP-CRL (c) achieves well-separated and compact clusters, demonstrating superior discriminative performance.

\begin{figure}[ht]
\vspace{-0.5em}
  \begin{center}
    \centerline{\includegraphics[width=\textwidth]{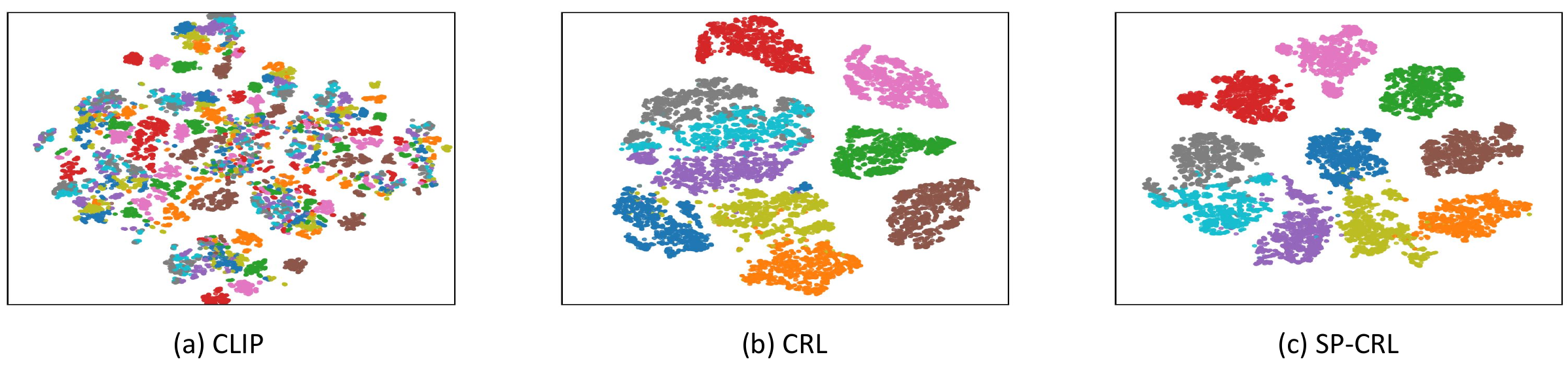}}
    \caption{T-SNE visualizations of the representations learned by CLIP, CRL, and SP-CRL on the Clevr4-10k dataset for the ``color" criterion.}
    \label{fig_5}
  \end{center}
\vspace{-0.5em}
\end{figure}

\section{Sensitivity Analysis}

\subsection{Sensitivity Analysis on Basis Quantity.}
\label{Basis_Quantity}
In practical scenarios, the optimal number of text basis vectors remains unknown. To validate the robustness of SP-CRL to text basis quantity, we examine its performance on the customized clustering task under varying numbers of basis vectors. As shown in \cref{fig_6}, we conduct experiments on both Clevr4-10k and Cards datasets with the number of basis vectors ranging from 5 to 2000. SP-CRL exhibits stable and superior performance across all criteria, remaining insensitive to the number of basis vectors. In contrast, CRL demonstrates significant performance variations, particularly in color-based clustering on Clevr4-10k. Moreover, regardless of the initial basis cardinality, SP-CRL utilizes only a small set of refined basis vectors constructed through Criterion-Specific Basis Refinement. This enables SP-CRL to achieve significant efficiency in scenarios where the semantics corresponding to the criterion are overly broad.

\begin{figure}[ht]
\vspace{-0.5em}
  \begin{center}
    \centerline{\includegraphics[width=\columnwidth]{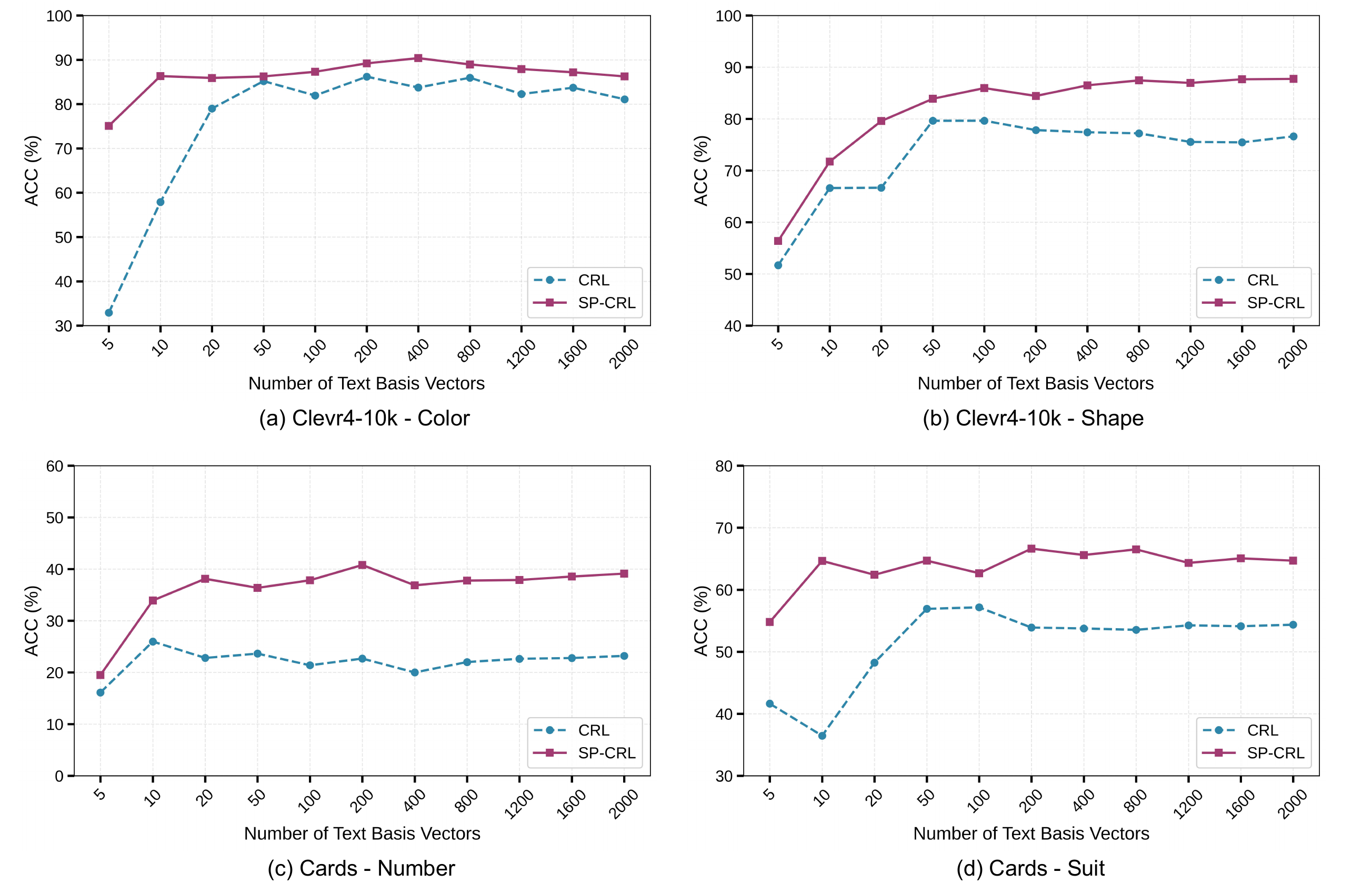}}
    \caption{Performance with different numbers of basis vectors.}
    \label{fig_6}
  \end{center}
\vspace{-0.5em}
\end{figure}

\subsection{Sensitivity Analysis on LLM Selection.}
\label{LLM_Selection}
SP-CRL leverages criterion-specific descriptions generated by LLMs to construct text bases. By default, we adopt GPT-4 \citep{achiam2023gpt} as the LLM for description generation. To validate the robustness of SP-CRL to LLM selection, we evaluate its performance on customized clustering across a variety of LLMs, including Claude 4, Gemini 2.5 \citep{comanici2025gemini}, and DeepSeek-v3 \citep{liu2024deepseek}. As shown in \cref{tab5}, experiments are conducted on the Clevr4-10k and Cards datasets, covering the Shape, Color, Number, and Suit criteria. Under different LLM selections, SP-CRL maintains stable performance. The average standard deviation across all metrics is only 0.89\%, demonstrating the robustness of SP-CRL to the choice of description generator.
\begin{table}[ht]
\vspace{-0.5em}
  \caption{Performance under Different LLM Selections.}
  \label{tab5}
  \begin{center}
    \begin{footnotesize}
    \setlength{\tabcolsep}{3.5pt}
        \begin{tabular}{c|ccc|ccc|ccc|ccc}
          \toprule
          & \multicolumn{6}{c|}{Clevr4-10k} & \multicolumn{6}{c}{Cards}\\
          \cmidrule{2-13}
          LLM & \multicolumn{3}{c|}{Shape} & \multicolumn{3}{c|}{Color} & \multicolumn{3}{c|}{Number} & \multicolumn{3}{c}{Suit}\\
          \cmidrule{2-13}
          & NMI & ACC & ARI & NMI & ACC & ARI & NMI & ACC & ARI & NMI & ACC & ARI\\
          \midrule
          GPT-4       & 86.67 & 90.69 & 84.11 & 89.99 & 89.88 & 85.09 & 38.78 & 44.45 & 28.40 & 41.61 & 73.22 & 42.20\\
          Claude 4    & 86.54 & 88.94 & 83.36 & 89.73 & 90.94 & 85.11 & 38.45 & 42.79 & 27.54 & 40.60 & 72.56 & 41.26\\
          Gemini 2.5  & 85.29 & 88.49 & 82.12 & 88.71 & 88.65 & 82.86 & 37.66 & 40.42 & 25.70 & 41.58 & 71.68 & 41.85\\
          DeepSeek-v3 & 85.14 & 87.57 & 81.29 & 88.67 & 87.65 & 82.32 & 38.07 & 42.37 & 26.37 & 40.15 & 71.96 & 40.70\\
          \bottomrule
        \end{tabular}
    \end{footnotesize}
  \end{center}
\vspace{-0.5em}
\end{table}

\subsection{Sensitivity Analysis on VLM Selection.}
\label{VLM_Selection}
SP-CRL relies on vision-language models to encode both image embeddings and text bases. To evaluate the robustness of SP-CRL to VLM selection, we conduct customized clustering experiments using CLIP \citep{radford2021learning}, BLIP2 \citep{li2023blip}, and SigLIP \citep{zhai2023sigmoid} as the underlying models. As shown in \cref{tab13}, experiments are conducted on the Clevr4-10k and Cards datasets, covering the Shape, Color, Number, and Suit criteria. Under different VLM selections, SP-CRL consistently outperforms both the original VLMs and CRL on most evaluation metrics. For example, on Clevr4-10k, CLIP+SP-CRL improves the Shape NMI from 78.40\% to 86.67\% over CLIP+CRL, while on Cards, BLIP2+SP-CRL improves the Number ACC from 48.55\% to 61.90\%. These results demonstrate the strong generalization of SP-CRL across different VLM backbones.

\begin{table}[ht]
\vspace{-0.5em}
  \caption{Performance under Different VLM Selections.}
  \label{tab13}
  \begin{center}
    \begin{footnotesize}
    \setlength{\tabcolsep}{3.3pt}
    \begin{tabular}{c|ccc|ccc|ccc|ccc}
      \toprule
      & \multicolumn{6}{c|}{Clevr4-10k} & \multicolumn{6}{c}{Cards}\\
      \cmidrule{2-13}
      Method & \multicolumn{3}{c|}{Shape} & \multicolumn{3}{c|}{Color} & \multicolumn{3}{c|}{Number} & \multicolumn{3}{c}{Suit}\\
      \cmidrule{2-13}
      & NMI & ACC & ARI & NMI & ACC & ARI & NMI & ACC & ARI & NMI & ACC & ARI\\
      \midrule
      CLIP & 74.22 & 73.19 & 64.15 & 0.83 & 12.23 & 0.27 & 16.84 & 18.91 & 8.44 & 16.52 & 43.74 & 12.93\\
      CLIP+CRL & \underline{78.40} & \underline{81.83} & \underline{71.35} & \underline{88.09} & \underline{87.26} & \underline{81.07} & \underline{24.68} & \underline{28.59} & \underline{12.27} & \underline{37.70} & \underline{62.08} & \underline{34.36}\\
      \textbf{CLIP+SP-CRL} & \textbf{86.67} & \textbf{90.69} & \textbf{84.11} & \textbf{89.99} & \textbf{89.88} & \textbf{85.09} & \textbf{38.78} & \textbf{44.45} & \textbf{28.40} & \textbf{41.61} & \textbf{73.22} & \textbf{42.20}\\
      \midrule
      BLIP2 & 86.98 & 85.68 & 81.17 & 0.99 & 11.92 & 0.24 & 24.34 & 25.25 & 13.08 & 31.26 & 47.04 & 22.25\\
      BLIP2+CRL & \underline{90.01} & \underline{87.97} & \underline{83.35} & \underline{83.44} & \underline{81.78} & \underline{74.37} & \underline{46.40} & \underline{48.55} & \underline{33.47} & \underline{59.43} & \underline{71.96} & \underline{53.00}\\
      \textbf{BLIP2+SP-CRL} & \textbf{93.20} & \textbf{93.01} & \textbf{90.26} & \textbf{88.86} & \textbf{90.39} & \textbf{83.68} & \textbf{58.07} & \textbf{61.90} & \textbf{48.76} & \textbf{62.52} & \textbf{83.33} & \textbf{62.71}\\
      \midrule
      SigLIP & 89.75 & 88.24 & 84.51 & 0.41 & 11.61 & 0.08 & 36.89 & 34.51 & 22.79 & 38.19 & 51.23 & 29.00\\
      SigLIP+CRL & \underline{91.56} & \textbf{91.64} & \underline{87.78} & \underline{83.39} & \underline{80.19} & \underline{72.83} & \underline{78.43} & \underline{72.75} & \underline{63.77} & \underline{71.88} & \underline{75.95} & \underline{65.38}\\
      \textbf{SigLIP+SP-CRL} & \textbf{93.43} & \underline{90.81} & \textbf{87.96} & \textbf{88.78} & \textbf{85.62} & \textbf{81.30} & \textbf{87.52} & \textbf{89.12} & \textbf{84.56} & \textbf{80.36} & \textbf{93.86} & \textbf{84.03}\\
      \bottomrule
    \end{tabular}
    \end{footnotesize}
  \end{center}
\vspace{-0.5em}
\end{table}

\end{document}